\documentclass[10pt]{article} 
\usepackage[preprint]{tmlr}

\usepackage{amsmath,amsfonts,bm}

\def\eqref#1{equation~\ref{#1}}

\def\1{\bm{1}}

\DeclareMathAlphabet{\mathsfit}{\encodingdefault}{\sfdefault}{m}{sl}
\SetMathAlphabet{\mathsfit}{bold}{\encodingdefault}{\sfdefault}{bx}{n}

\usepackage{hyperref}
\usepackage{url}
\usepackage{algorithm}
\usepackage{algorithmic}
\usepackage{algorithm}
\usepackage{algorithmic}
\usepackage{microtype}
\usepackage{graphicx}
\usepackage{amsmath, amssymb, amsfonts, amsthm}
\usepackage{booktabs}
\usepackage{arydshln}
\usepackage{adjustbox}
\usepackage{tcolorbox}
\usepackage[table,dvipsnames]{xcolor}
\usepackage{mathtools}
\usepackage{multirow}

\newtheorem{theorem}{Theorem}
\newtheorem{lemma}{Lemma}

\newtheorem{definition}{Definition}

\newtheorem{assumption}[theorem]{Assumption}

\definecolor{tablerowcolor}{gray}{0.95}

\title{Sparse Semantic Dimension as a Generalization Certificate for LLMs}

\author{\name Dibyanayan Bandyopadhyay \email dibyanayan\_2321cs14@iitp.ac.in \\
      \addr Department of Computer Science and Engineering\\
      Indian Institute of Technology Patna
      \AND
      \name Asif Ekbal \email asif@iitp.ac.in \\
      \addr Department of Computer Science and Engineering\\
      Indian Institute of Technology Patna
     }

\def\month{MM}  
\def\year{YYYY} 
\def\openreview{\url{https://openreview.net/forum?id=XXXX}} 

\begin{document}

\maketitle

\begin{abstract}
Standard statistical learning theory predicts that Large Language Models (LLMs) should overfit because their parameter counts vastly exceed the number of training tokens. Yet, in practice, they generalize robustly. We propose that the effective capacity controlling generalization lies in the geometry of the model's internal representations: while the parameter space is high-dimensional, the activation states lie on a low-dimensional, sparse manifold. To formalize this, we introduce the Sparse Semantic Dimension (SSD), a complexity measure derived from the active feature vocabulary of a Sparse Autoencoder (SAE) trained on the model's layers. Treating the LLM and SAE as frozen oracles, we utilize this framework to attribute the model's generalization capabilities to the sparsity of the dictionary rather than the total parameter count. Empirically, we validate this framework on GPT-2 Small and Gemma-2B, demonstrating that our bound provides non-vacuous certificates at realistic sample sizes. Crucially, we uncover a counter-intuitive “feature sharpness” scaling law: despite being an order of magnitude larger, Gemma-2B requires significantly fewer calibration samples to identify its active manifold compared to GPT-2, suggesting that larger models learn more compressible, distinct semantic structures. Finally, we show that this framework functions as a reliable safety monitor: out-of-distribution inputs trigger a measurable “feature explosion” (a sharp spike in active features), effectively signaling epistemic uncertainty through learned feature violation. Code is available at: \url{https://github.com/newcodevelop/sparse-semantic-dimension}.
\end{abstract}

\section{Introduction}

The precipitous rise of Large Language Models (LLMs) has engendered a fundamental crisis in statistical learning theory, often referred to as the \textit{Generalization Paradox} \citep{zhang2017understandingdeeplearningrequires}. Modern foundational models operate in a regime where the number of parameters, $P_{\theta}$, vastly exceeds the number of training tokens, $N$ (i.e., $P_{\theta} \gg N$). According to classical measures of complexity—such as VC-dimension or Rademacher complexity applied to weight matrices—these models possess the capacity to memorize random noise and should, theoretically, exhibit severe overfitting \citep{nagarajan2021uniformconvergenceunableexplain}. Yet, in practice, LLMs not only memorize training data but generalize robustly to unseen datasets. When standard uniform convergence bounds are applied to architectures like GPT-2 or Gemma, they typically yield ``vacuous'' results (predicting generalization errors exceeding $100\%$), failing to provide a mathematical certificate for the observed performance \citep{dziugaite2017computingnonvacuousgeneralizationbounds, lotfi2024nonvacuousgeneralizationboundslarge}.

Resolving this discrepancy requires shifting the measure of complexity from the high-dimensional, static parameter space to a data-dependent, low-dimensional manifold. This intuition is rooted in the Minimum Description Length (MDL) principle and Occam’s Razor: generalization arises not from the sheer number of parameters, but from the model's ability to compress the data into a concise representation \citep{RISSANEN1978465, BLUMER1987377}. While prior works have attempted to bound generalization via weight compression~\citep{arora2018strongergeneralizationboundsdeep}, we posit that the essential compressibility of an LLM lies in the sparsity of its \textit{internal representations}, rather than its parameters. To operationalize this, we utilize Sparse Autoencoders (SAEs) to investigate the model's latent structure. We hypothesize that although the parameter space is high-dimensional, the internal representations are confined to a low-dimensional, sparse \textit{manifold}, defined as a subset of the activation space where most features are zero~\citep{ansuini2019intrinsicdimensiondatarepresentations}. Consequently, the \textit{effective complexity} of the LLM is governed not by the massive parameter count, but by the sparse set of active features that define this manifold. This choice also brings an advantage that was absent in earlier non-vacuous generalization analyses for LLMs \citep{lotfi2024nonvacuousgeneralizationboundslarge}. Prior works relied on weight compression to learn sparse parameter matrices in order to obtain bounds with realistic sample sizes. However, this approach incurs substantial performance degradation in the underlying LLM. In contrast, our method avoids modifying the LLM parameters altogether. 
Instead, we derive a generalization bound using a composite proxy model—comprising the frozen LLM and an SAE trained on its activations—which allows us to certify the original model's performance without degradation. Crucially, we adopt a post-hoc setting where the LLM and SAE are treated as fixed "frozen oracles" prior to evaluation. While a standard Hoeffding bound on the frozen model's loss would be numerically tighter, it provides no structural insight; our framework is necessary to explicitly link generalization to the compressibility of internal representations rather than raw parameter count, which would otherwise yield vacuous bounds due to the massive over-parameterization of modern LLMs.



To formalize this hypothesis, we employ sparse representations learned by SAEs not merely as visualization tools, but as a rigorous mechanism to bound the underlying LLM's generalization performance. We propose that if dense LLM representations can be faithfully approximated by a sparse autoencoder with a concise dictionary of active concepts, then the generalization error is governed by the sparsity of that dictionary rather than the dimension of the model parameters.

We introduce a new complexity measure, the \textit{Sparse Semantic Dimension} (SSD), derived from the active feature vocabulary of the SAE:
\begin{equation}
    SSD := P \log\left(\frac{em}{P}\right)
\end{equation}
where $P$ is the size of the global active feature pool (the distinct features activated across the dataset) and $m$ is the total dictionary size (where typically $P < m$). By decomposing the base LLM model into this sparse proxy and a bounded reconstruction gap ($\epsilon_{loss}$), we derive a structural risk minimization bound. Unlike parameter-counting bounds, which scale with $P_{\theta} \approx 10^9$, our bound scales with the effective feature pool $P \approx 10^4$, theoretically yielding non-vacuous certificates at realistic sample sizes.

Our contributions are as follows:
\begin{itemize}
    \item \textbf{Sparse Semantic Generalization Bound:} We derive a high-probability sequence level bound for LLM that depends on the active sparsity $P$ of the SAE decomposition. We show that this bound provides a mathematically valid certificate of generalization for models where the parameter space suggests overfitting.
    \item \textbf{Empirical Non-Vacuousness:} We validate our theory on GPT-2 Small and Gemma-2B. We demonstrate that our bound becomes non-vacuous (tighter than a random guess) with as few as $N \approx 50,000$ evaluation samples, offering a significant improvement over classical vacuous bounds.

    \item \textbf{Uncertainty Quantification via Feature Density:} We identify a practical application of our theory in Out-Of-Distribution (OOD) detection. We observe that generalization is intrinsically linked to the sparsity of the active feature set. When models encounter far-OOD inputs (e.g., random noise), this sparsity pattern collapses, leading to a measurable ``feature explosion'' (a sharp spike in the active feature count $k$). This suggests that runtime sparsity can serve as a reliable, computationally cheap proxy for epistemic uncertainty.

\end{itemize}

By establishing a theoretical link between the compressibility of internal representations and statistical guarantees, this work posits that the emergence of sparse, interpretable features is not merely an empirical artifact, but a structural necessity for robust generalization in high-dimensional parameter spaces.

\section{Related Work}

This work bridges two distinct subfields of deep learning research: Mechanistic Interpretability, specifically the study of Sparse Autoencoders (SAEs), and Statistical Learning Theory, focusing on generalization bounds for over-parameterized models.

\subsection{Mechanistic Interpretability and Sparse Autoencoders}
The field of Mechanistic Interpretability aims to reverse-engineer neural networks into human-understandable components. A primary challenge in analyzing Large Language Models (LLMs) is \textit{polysemanticity}, where single neurons activate for unrelated concepts, and \textit{superposition}, where models represent more features than dimensions by storing them non-orthogonally \citep{elhage2022toymodelssuperposition}.

To resolve these entangled representations, \textbf{Sparse Autoencoders (SAEs)} have emerged as a scalable unsupervised probing framework \citep{bricken2023monosemanticity}. SAEs learn an overcomplete, sparse basis that decomposes dense activation vectors into interpretable "semantic" directions \citep{cunningham2023sparseautoencodershighlyinterpretable}. Recent work has successfully scaled SAEs to real-world models like GPT-2 and Gemma, revealing high-fidelity features that govern model behavior. Our work adopts the SAE framework not merely for qualitative inspection, but as a formal tool to define the \textit{Sparse Semantic Dimension} (SSD), leveraging the hypothesis that the activation space operates on a low-dimensional manifold despite the high dimensionality of the parameter space, leading to generalisation behaviour of LLMs. Low-dimensional complexity to explain generalisation behaviour of LLMs is an active research idea \citep{lotfi2024nonvacuousgeneralizationboundslarge}.

\subsection{The Generalization Paradox in Large Language Models}
Statistical learning theory faces a fundamental crisis often termed the "Generalization Paradox" \citep{zhang2017understandingdeeplearningrequires}. Modern LLMs operate in a regime where the number of parameters $P_{\theta}$ vastly exceeds the number of training tokens $N$ ($P_{\theta} \gg N$). Under classical complexity measures such as VC-dimension or Rademacher complexity, such models should overfit massively, yet they generalize robustly \citep{lotfi2024nonvacuousgeneralizationboundslarge}.

Standard uniform convergence bounds applied to deep networks typically yield "vacuous" results (predicting generalization error $>100\%$) because they penalize the raw parameter count \citep{dziugaite2017computingnonvacuousgeneralizationbounds, nagarajan2021uniformconvergenceunableexplain}. While recent advances have proposed norm-based bounds or sharpness-aware minimization to explain this phenomenon, a few approaches successfully account for the specific structure of natural language data applied to Large Language Models. Our work addresses this gap by shifting the complexity measure from the static parameter space $P_{\theta}$ to the data-dependent sparse active feature pool $P$, proposing that the effective complexity is governed by sparsity rather than raw weight count.



\subsection{Compression and Occam's Razor in Deep Learning}
Our theoretical approach is grounded on the principle that generalization arises from the compressibility of internal representations, formally linking interpretability to Occam's Razor bounds \citep{RISSANEN1978465, BLUMER1987377}. While classical learning theory applies complexity measures to high-dimensional parameter spaces ($P_{\theta}$), often resulting in vacuous bounds \citep{zhang2017understandingdeeplearningrequires, dziugaite2017computingnonvacuousgeneralizationbounds}, we propose that the effective complexity of LLMs is governed by a low-dimensional, sparse manifold within the activation space \citep{ansuini2019intrinsicdimensiondatarepresentations, pope2021intrinsicdimensionimagesimpact}.

To formalize this, we decompose the dense predictor $M$ into a sparse proxy $S \circ M$ (where $S$ is the SAE) and a bounded reconstruction gap $\epsilon_{loss}$. This decomposition enables a structural risk minimization bound where the leading complexity term scales with the Sparse Semantic Dimension (SSD)---defined by the active feature count $P$, rather than the total parameter count. This aligns with prior work on compression-based bounds \citep{arora2018strongergeneralizationboundsdeep, lotfi2024nonvacuousgeneralizationboundslarge, lotfi2024unlockingtokensdatapoints} but diverges by compressing the \textit{internal representations} via a fixed concept pool. Empirically, this connection is validated in Out-Of-Distribution (OOD) settings: when the sparse manifold assumption is violated, models exhibit a ``feature explosion'' (a sharp spike in per-input active features $k$), which serves as a quantifiable proxy for epistemic uncertainty \citep{hendrycks2018baselinedetectingmisclassifiedoutofdistribution}.

\section{Preliminaries}\label{prelim}

Before the generalization bound, we briefly review the mechanics of Sparse Autoencoders (SAEs) and outline our theoretical approach. This section establishes the notation used throughout the proofs and clarifies the distinct, post-hoc nature of our analysis.

\subsection{Sparse Autoencoder Notation}
We analyze a base Large Language Model (LLM), denoted as $M$, which maps an input $x$ to a high-dimensional hidden representation $h(x) \in \mathbb{R}^d$ at a specific layer. To interpret this dense representation, we utilize a Sparse Autoencoder (SAE) $S$, consisting of an encoder $S_E: \mathbb{R}^d \rightarrow \mathbb{R}^m$ and a decoder $S_D: \mathbb{R}^m \rightarrow \mathbb{R}^d$, where the dictionary size $m$ is typically much larger than the model width $d$ ($m \gg d$).

The SAE decomposes the activation into a sparse set of interpretable features via the following operations:

i) \textbf{Encoding:} The dense hidden state is projected to a pre-activation feature vector $a(x) := S_E(h(x)) \in \mathbb{R}^m$.

ii) \textbf{Sparsification:} We apply a non-linear $TopK$ operator, which retains the $k$ coefficients with the largest magnitudes and sets the rest to zero. This yields the sparse code $c(x) := TopK(a(x))$, satisfying $\|c(x)\|_0 \le k$.

iii) \textbf{Reconstruction:} The sparse code is mapped back to the original activation space to produce the approximate hidden state $\hat{h}(x) := S_D(c(x))$.

We denote the full proxy predictor, where the internal activation of $M$ ($h(x)$) is replaced by the SAE reconstruction ($\hat{h}(x)$), as $S \circ M$.

\subsection{Overview of the Theoretical Approach}
The goal of Section \ref{prelim} is to derive a generalization certificate for the base model $M$. It is crucial to note that our framework operates in a \textit{post-hoc} regime, distinct from classical learning theory bounds that predict generalization from initialization.

Our analysis proceeds in two phases:
i) \textbf{Phase 1 (Freezing):} The base model $M$ and the SAE components ($S_E, S_D$) are pre-trained and fixed. For the purpose of our theorem, they are treated as frozen oracles, not as variable hypotheses.
ii) \textbf{Phase 2 (Certification):} On a held-out dataset, we identify the optimal subset of dictionary features (the ``concept pool'') that covers the hypothesis space.

By treating the model weights as fixed constants, we shift the complexity measure from the continuous parameter space of $M$ to the discrete feature activation of the SAE dictionary. This allows us to bound the generalization error by the sparsity of the active concept pool, providing a valid certificate of performance even for models with infinite effective capacity.

\section{Problem Definition}

We address the \textit{Generalization Paradox} of Large Language Models (LLMs): how models with parameters $P \gg N$ (where $N$ is the number of training tokens) avoid overfitting. We hypothesize that while the \textit{parameter space} is high-dimensional, the \textit{activation space} operates on a low-dimensional sparse manifold.

\subsection{Risk Formulation}
Let $\mathcal{X}$ be the input space and $\mathcal{D}$ be an unknown distribution over $\mathcal{X}$. In the language modeling setting, we take a sample to be a token sequence $x = x_{1:T}$. We define the population risk as:
\begin{equation}
    \mathcal{R}(M) := \mathbb{E}_{x_{1:T} \sim \mathcal{D}}\big[\ell(M, x_{1:T})\big]
\end{equation}
and, given $N$ i.i.d. samples $\{x^{(i)}_{1:T}\}_{i=1}^N$, the empirical risk is defined as:
\begin{equation}
    \hat{\mathcal{R}}_N(M) := \frac{1}{N}\sum_{i=1}^N \ell\big(M, x^{(i)}_{1:T}\big).
\end{equation}
The loss $\ell$ is assumed to take values in $[0,B]$. We make use of prediction smoothing for bounding the loss as elaborated in Section \ref{blf}. Note that the token sequences must be i.i.d samples for the bound to hold. For that, we break the sequences into a contiguous set of tokens and then sample the sequences uniformly randomly from the dataset. This is the approach also used by \citet{lotfi2024nonvacuousgeneralizationboundslarge}.

\subsection{The Bounded Loss Function (Smoothed BPD)} \label{blf}
In the language modeling setting, let the model induce next-token probabilities $p_M(\cdot\mid x_{<t})$ over a vocabulary of size $V$. The standard bits-per-dimension (BPD) loss is:
\begin{equation}
    \ell_{\mathrm{bpd}}(M, x_{1:T}) := -\frac{1}{T}\sum_{t=1}^{T} \log_2 p_M(x_t\mid x_{<t}).
\end{equation}
Since $\ell_{\mathrm{bpd}}$ is unbounded when $p_M(x_t\mid x_{<t})$ can be arbitrarily small, we use \emph{prediction smoothing}. For a fixed $\alpha \in (0,1)$, as first proposed and defined in \cite{lotfi2024nonvacuousgeneralizationboundslarge}:
\begin{equation}
    \tilde p_M(\cdot\mid x_{<t}) := (1-\alpha) p_M(\cdot\mid x_{<t}) + \alpha/V.
\end{equation}
We then define the smoothed BPD loss:
\begin{equation}
    \ell(M, x_{1:T}) := -\frac{1}{T}\sum_{t=1}^{T} \log_2 \tilde p_M(x_t\mid x_{<t}).
\end{equation}

This loss is bounded because $\tilde p_M(x_t\mid x_{<t}) \ge \alpha/V$, hence $\log_2(V/\alpha) - \Delta\le \ell(M, x_{1:T}) \le \log_2(V/\alpha) =: B$, where $\Delta = \log_2 (1 + (1 - \alpha)V/\alpha)$. For rigorous derivation, check Appendix A.2 of \cite{lotfi2024nonvacuousgeneralizationboundslarge}.

\subsection{The Sparse Autoencoder (SAE) based Generalization Framework}
To formalize the complexity of $M$, we introduce a \textbf{Sparse Autoencoder (SAE)} probe, denoted $S$. The proxy predictor $S \circ M$ is defined by replacing the internal activation $h_t = M(x_{<t})$ with its SAE reconstruction $\tilde h_t = S(h_t) = S(M(x_{<t}))$, which is then fed through the same downstream layers to produce smoothed proxy probabilities $\tilde p_{S\circ M}$.

\begin{definition}[\textbf{Sparse Autoencoder Class}]
Let $\mathcal{H}_{k,m}$ be the class of functions realizable by an SAE with dictionary $D \in \mathbb{R}^{d \times m}$ (with unit-norm columns) and sparsity constraint $k$. For any input $x$, the output is:
\begin{equation}
    S(x) = D \cdot c(x)
\end{equation}
where $\|c(x)\|_0 \le k$. The SAE effectively compresses the dense activation $M(x)$ into a sparse code $c$.
\end{definition}

\begin{definition}[\textbf{Reconstruction Inefficiency}]
We define a \emph{loss-level} reconstruction gap $\epsilon_{loss}$ as the expected discrepancy in loss between the original predictor and the proxy predictor:
\begin{equation}
    \epsilon_{loss} = \mathbb{E}_{x_{1:T} \sim \mathcal{D}}\big[\,|\ell(M, x_{1:T}) - \ell(S \circ M, x_{1:T})|\,\big].
\end{equation}
\end{definition}

To rigorously bound the complexity, we formalize the hypothesis space of the sparse proxy.

\subsection{Setup and notation}
\label{sec:setup}

Let $\mathcal{D}$ be a distribution over labeled examples $z=(x,y)\in\mathcal{Z}$.
Let $S=\{z_i\}_{i=1}^N$ be an i.i.d.\ sample, $S\sim\mathcal{D}^N$.

We consider a base predictor $M$ and a proxy predictor $S\circ M$ constructed by inserting a sparse autoencoder (SAE)
bottleneck at a fixed hidden layer of $M$.

\paragraph{Hidden representation and SAE.}
Fix a layer of $M$ that maps an input $x$ to a hidden representation $h(x)\in\mathbb{R}^d$.
Let $S_E:\mathbb{R}^d\to\mathbb{R}^m$ be an SAE encoder and $S_D:\mathbb{R}^m\to\mathbb{R}^d$ be an SAE decoder.
Define the (dense) concept activation vector
\begin{equation}
a(x) := S_E(h(x)) \in \mathbb{R}^m.
\end{equation}

\paragraph{Top-$k$ operator.}
Fix $k\in\{1,\dots,m\}$.
Let $\mathrm{TopK}:\mathbb{R}^m\to\mathbb{R}^m$ be the deterministic operator that keeps the $k$ coordinates
of largest magnitude and sets all other coordinates to $0$ (ties are broken deterministically, e.g.\ by preferring smaller indices).
Define the sparse code
\begin{equation}
c(x) := \mathrm{TopK}(a(x)) \in \mathbb{R}^m.
\end{equation}
The reconstructed hidden state is
\begin{equation}
\widehat{h}(x) := S_D(c(x)) \in \mathbb{R}^d.
\end{equation}

\paragraph{Proxy predictor.}
The proxy predictor $S\circ M$ is obtained by feeding $\widehat{h}(x)$ into the downstream part of $M$
(from the insertion layer onward) to produce a predictive distribution over outputs.
We write $(S\circ M)(x)$ for the resulting predictive distribution.

\paragraph{Loss and risk.}
Let $\ell(\cdot;z)$ denote the loss of a predictor on example $z=(x,y)$.
We assume bounded loss:
\begin{assumption}[Bounded loss]
\label{ass:bounded_loss}
There exists $B>0$ such that for all predictors considered and all $z\in\mathcal{Z}$,
\begin{equation}
B-\Delta \le \ell(\cdot;z) \le B.
\end{equation}
\end{assumption}
For a deterministic predictor $f$, define population and empirical risks
\begin{equation}
R(f) := \mathbb{E}_{z\sim\mathcal{D}}[\ell(f;z)],
\qquad
\widehat{R}_S(f) := \frac{1}{N}\sum_{i=1}^N \ell(f;z_i).
\end{equation}

\paragraph{Approximation error between $M$ and $S\circ M$.}
Define the pointwise loss gap
\begin{equation}
\Delta_{\mathrm{loss}}(z) := \big|\ell(M;z)-\ell(S\circ M;z)\big| \in [0,\Delta],
\end{equation}
and its population and empirical means
\begin{equation}
\epsilon_{\mathrm{loss}} := \mathbb{E}_{z\sim\mathcal{D}}[\Delta_{\mathrm{loss}}(z)],
\qquad
\widehat{\epsilon}_{\mathrm{loss}} := \frac{1}{N}\sum_{i=1}^N \Delta_{\mathrm{loss}}(z_i).
\end{equation}

\subsection{Concept-pool assumption and restricted proxy class}
\label{sec:pool}

For an index set $G\subseteq[m]:=\{1,\dots,m\}$, let $\mathbf{1}_G\in\{0,1\}^m$ denote its indicator vector.
Define the masked activation vector
\begin{equation}
a_G(x) := a(x)\odot \mathbf{1}_G.
\end{equation}
Define the pool-restricted code and reconstruction
\begin{equation}
c_G(x) := \mathrm{TopK}(a_G(x)),
\qquad
\widehat{h}_G(x) := S_D(c_G(x)).
\end{equation}
Let $h_G$ denote the predictor obtained by feeding $\widehat{h}_G(x)$ into the downstream layers of $M$,
analogously to $S\circ M$.
Thus $\{h_G\}$ is a family of proxy predictors indexed by pools $G$.

\paragraph{Top-$k$ support.}
Let $\mathrm{supp}(v):=\{j: v_j\neq 0\}$.
Define the top-$k$ support event with respect to a pool $G$:
\begin{equation}
E_G(x) := \Big\{\mathrm{supp}\big(\mathrm{TopK}(a(x))\big)\subseteq G\Big\}.
\end{equation}

\begin{assumption}[Concept-pool assumption]
\label{ass:pool}
There exist integers $P\in\{k,\dots,m\}$ and a subset $G^\star\subseteq[m]$ with $|G^\star|=P$ such that
\begin{equation}
\Pr_{x\sim\mathcal{D}}\big(E_{G^\star}(x)\big) \ge 1-\eta,
\end{equation}
for some $\eta\in[0,1]$.
\end{assumption}

\paragraph{Restricted hypothesis class.}
Define
\begin{equation}
\mathcal{H}_P := \{h_G : G\subseteq[m],\ |G|=P\}.
\end{equation}
Then $|\mathcal{H}_P|=\binom{m}{P}$.

\subsection{Auxiliary lemmas}
\label{sec:lemmas}

\begin{lemma}[Decomposition of risk via loss gap]
\label{lem:decomp}
For any two predictors $f,g$,
\begin{equation}
R(f) \le R(g) + \mathbb{E}_{z\sim\mathcal{D}}\big|\ell(f;z)-\ell(g;z)\big|.
\end{equation}
In particular,
\begin{equation}
R(M) \le R(S\circ M) + \epsilon_{\mathrm{loss}}.
\end{equation}
\end{lemma}
\begin{proof}
For every $z$, $\ell(f;z)\le \ell(g;z) + |\ell(f;z)-\ell(g;z)|$. Take expectation over $z\sim\mathcal{D}$.
\end{proof}

\begin{lemma}[Pool mismatch bound]
\label{lem:pool_mismatch}
Under Assumption~\ref{ass:pool}, define
\begin{equation}
\epsilon_{\mathrm{pool}} := \mathbb{E}_{z=(x,y)\sim\mathcal{D}}\big|\ell(S\circ M;z)-\ell(h_{G^\star};z)\big|.
\end{equation}
Then
\begin{equation}
\epsilon_{\mathrm{pool}} \le \eta B.
\end{equation}
\end{lemma}
\begin{proof}
Fix $x$. If $E_{G^\star}(x)$ holds, then $\mathrm{TopK}(a(x))$ has support contained in $G^\star$, hence
$a_{G^\star}(x)$ agrees with $a(x)$ on all coordinates that survive $\mathrm{TopK}$ and is zero elsewhere.
By determinism of $\mathrm{TopK}$ and the fixed tie-breaking, we obtain $c_{G^\star}(x)=c(x)$ and therefore
$\widehat{h}_{G^\star}(x)=\widehat{h}(x)$. Thus, $(S\circ M)(x)=h_{G^\star}(x)$, implying
$\ell(S\circ M;z)=\ell(h_{G^\star};z)$ for all labels $y$.
If $E_{G^\star}(x)$ fails, the absolute loss difference is at most $B$ by Assumption~\ref{ass:bounded_loss}.
Therefore, for all $z=(x,y)$,
\begin{equation}
\big|\ell(S\circ M;z)-\ell(h_{G^\star};z)\big|
\le B\cdot \mathbf{1}\{E_{G^\star}(x)^c\}.
\end{equation}
Taking expectation over $z\sim\mathcal{D}$ yields
\begin{equation}
\epsilon_{\mathrm{pool}}
\le
B\Pr_{x\sim\mathcal{D}}(E_{G^\star}(x)^c)
\le
\eta B.
\end{equation}
\end{proof}

\begin{lemma}[Uniform convergence for finite classes (Occam bound)]
\label{lem:occam}
Let $\mathcal{H}$ be a finite set of predictors and assume $\ell(h;z)\in[0,B]$ for all $h\in\mathcal{H}$ and $z$.
Then for any $\delta\in(0,1)$, with probability at least $1-\delta$ over $S\sim\mathcal{D}^N$,
\begin{equation}
\forall h\in\mathcal{H}:\quad
R(h) \le \widehat{R}_S(h) + B\sqrt{\frac{\ln|\mathcal{H}|+\ln(1/\delta)}{2N}}.
\end{equation}
\end{lemma}
\begin{proof}
Fix $h\in\mathcal{H}$. By Hoeffding's inequality applied to i.i.d.\ variables $\ell(h;z_i)\in[0,B]$,
\begin{equation}
\Pr\!\left(R(h) - \widehat{R}_S(h) > t\right) \le \exp\!\left(-\frac{2Nt^2}{B^2}\right).
\end{equation}
Apply a union bound over all $h\in\mathcal{H}$ and set the right-hand side to $\delta$ to solve for $t$.
\end{proof}


\begin{lemma}[Counting pools]
\label{lem:count}
For $\mathcal{H}_P$ defined above, $|\mathcal{H}_P|=\binom{m}{P}$ and
\begin{equation}
\ln|\mathcal{H}_P|
=
\ln\binom{m}{P}
\le
P\ln\!\left(\frac{em}{P}\right).
\end{equation}
\end{lemma}
\begin{proof}
The cardinality is the number of $P$-subsets of $[m]$.
The inequality uses $\binom{m}{P}\le (em/P)^P$.
\end{proof}

\begin{lemma}[Concentration of $\epsilon_{\mathrm{loss}}$]
\label{lem:epsloss_hoeffding}
Under Assumption~\ref{ass:bounded_loss}, for any $\delta\in(0,1)$, with probability at least $1-\delta$,
\begin{equation}
\epsilon_{\mathrm{loss}}
\le
\widehat{\epsilon}_{\mathrm{loss}}
+
B\sqrt{\frac{\ln(2/\delta)}{2N}}.
\end{equation}
\end{lemma}
\begin{proof}
The variables $\Delta_{\mathrm{loss}}(z_i)\in[0,B]$ are i.i.d.\ and Hoeffding applies.
\end{proof}
\subsection{Main theorems}
\label{sec:main}

\begin{theorem}[Generalization bound via compression (Occam) under a concept pool]
\label{thm:occam_pool}
Assume Assumption~\ref{ass:bounded_loss} and Assumption~\ref{ass:pool}.
Let $\delta\in(0,1)$ and define $\delta_1=\delta_2=\delta/2$.
Then with probability at least $1-\delta$ over $S\sim\mathcal{D}^N$,
\begin{equation}
R(M)
\le
\widehat{R}_S(h_{G^\star})
+
\widehat{\epsilon}_{\mathrm{loss}}
+
\eta B
+
B\sqrt{\frac{\ln|\mathcal{H}_P|+\ln(2/\delta)}{2N}}
+
B\sqrt{\frac{\ln(4/\delta)}{2N}}.
\label{eq:occam_final}
\end{equation}
Moreover, using Lemma~\ref{lem:count},
\begin{equation}
R(M)
\le
\widehat{R}_S(h_{G^\star})
+
\widehat{\epsilon}_{\mathrm{loss}}
+
\eta B
+
B\sqrt{\frac{P\ln(em/P)+\ln(2/\delta)}{2N}}
+
B\sqrt{\frac{\ln(4/\delta)}{2N}}.
\label{eq:occam_final_simplified}
\end{equation}
\end{theorem}
\begin{proof}


By Lemma~\ref{lem:decomp}, $R(M)\le R(S\circ M)+\epsilon_{\mathrm{loss}}$.
By Lemma~\ref{lem:pool_mismatch}, $R(S\circ M)\le R(h_{G^\star})+\epsilon_{\mathrm{pool}}\le R(h_{G^\star})+\eta B$.
Apply Lemma~\ref{lem:occam} to $\mathcal{H}_P$ with confidence $\delta_2$ and evaluate at $h_{G^\star}\in\mathcal{H}_P$:
\begin{equation}
R(h_{G^\star}) \le \widehat{R}_S(h_{G^\star}) + B\sqrt{\frac{\ln|\mathcal{H}_P|+\ln(1/\delta_2)}{2N}}.
\end{equation}
Apply Lemma~\ref{lem:epsloss_hoeffding} with confidence $\delta_1$:
\begin{equation}
\epsilon_{\mathrm{loss}} \le \widehat{\epsilon}_{\mathrm{loss}} + B\sqrt{\frac{\ln(2/\delta_1)}{2N}}.
\end{equation}
Combine the inequalities and use a union bound to obtain \eqref{eq:occam_final}.
Finally substitute $\ln|\mathcal{H}_P|\le P\ln(em/P)$ from Lemma~\ref{lem:count} to get \eqref{eq:occam_final_simplified}.
\end{proof}

\subsection{Discussions}

\paragraph{Remark on the Validity of the Finite Hypothesis Class.}
It is crucial to distinguish between training the high-dimensional parameter space of the base model (which corresponds to an effectively infinite hypothesis space) and post-hoc certification of generalization via a finite compressed family of hypotheses. In Theorem~\ref{thm:occam_pool}, we treat the dense predictor $M$ and the sparse autoencoder $\mathcal{S}$ as \emph{fixed} functional realizations (frozen oracles) chosen prior to drawing the evaluation sample $S\sim\mathcal{D}^N$. Under this formulation, the loss gap defining $\varepsilon_{\mathrm{loss}}$ is a bounded measurable random variable on $S$ and hence admits standard concentration (Lemma~\ref{lem:epsloss_hoeffding}) that does not depend on the parameter count $P_\theta$. Moreover, the only remaining search is over feature masks $G\subset[m]$ with $|G|=P$, yielding the finite class $\mathcal{H}_P=\{h_G:|G|=P\}$ with $|\mathcal{H}_P|=\binom{m}{P}$; therefore, the uniform Occam bound (Lemma~3), applied over all $h\in\mathcal{H}_P$, remains valid even though the underlying model $M$ has large capacity (and in particular it continues to hold for a data-chosen mask $\hat G$ (determined on $S$) as a consequence of the uniform guarantee).


\paragraph{Why SAE/SSD rather than Hoeffding on the frozen LLM?}
We acknowledge that if the sole objective is a tight post-hoc risk certificate for a \emph{fixed} model $M$ on an i.i.d.\ test stream, applying Hoeffding's inequality directly to $R(M)$ yields a tighter deviation term than Theorem~\ref{thm:occam_pool}. However, such a bound treats the model as a black box. The purpose of the SAE/SSD framework is to provide a \emph{structural} generalization certificate. By replacing vacuous parameter counts with a representation-level complexity measure (governed by the active concept pool size $P$), we decompose the generalization gap into interpretable components: empirical risk, reconstruction error, and complexity measures. This explanatory approach is also taken by \cite{lotfi2024nonvacuousgeneralizationboundslarge}. Furthermore, unlike a static concentration bound, the SSD quantities (specifically the active empirical per input feature count $k$) serve as dynamic, runtime signals for OOD detection and epistemic uncertainty, offering a safety monitor that is unavailable from a pure loss-based bound.


\paragraph{Why SAE features over native activations?}
We prioritize SAE features over native activations to ensure structural validity rather than numerical tightness. While a similar sparsity bound could theoretically apply to native model's ($M)$ neuronal representations, polysemanticity renders its diagnostics uninformative. Native neurons do not map to stable concepts and often remain densely active even under distribution shifts, failing to signal epistemic uncertainty. In contrast, SAEs enforce a sparse, semantically aligned basis where the assumption of a small active concept pool ($P \ll d$) is physically justified. Crucially, this framework links the theoretical complexity term ($SSD$) directly to the empirical per-sample sparsity ($k$). Under OOD inputs, we observe a "feature explosion" (larger $k$), which increases the correction terms and loosens the certificate precisely when the input deviates from the model's learned distribution. 

We also argue that utilizing $M$'s activation is unsuited for usage as a complexity term in the bound due to a trivial shortcoming. If we define the effective dimension as the activation representation dimension ($P = d_{model}$), the complexity term vanishes $\binom{P}{m} \rightarrow 1$, collapsing the bound to a standard Hoeffding's inequality with standard error terms. While valid, this result is trivial: it confirms that the model generalizes but fails to explain why, offering no insight into the relevance of compression for generlaization. Conversely, if we attempt to restrict the native dimension ($P \ll d_{model}$) to demonstrate compression, the bound may become vacuous. Due to the polysemantic and nature of native neuron, where information is distributed diffusely across the layer (superposition)—removing any subset of neurons may cause a catastrophic rise in reconstruction error ($\epsilon_{loss}$).

The SAE framework resolves this by introducing tunability of $P$ for getting the strictest bound. Because SAE features are monosemantic and sparse, the reconstruction error does not blow up even if we choose $P < m$\footnote{Note that we can choose $P$ on the evaluation set that minimizes the bound, thanks to the union bound.} This capacity to optimize $P$, to tune $P$ without non-vacuity, which native representations cannot offer, is what makes the SAE a prerequisite for deriving non-vacuous, explanatory generalization certificates. The core of the paper is not to prove that training an SAE results in generalization. \emph{The study establishes that if a Sparse Autoencoder (SAE) exhibits a lower Sparse Semantic Dimension (SSD) on new data, it serves as a mathematical certificate that the underlying model generalizes robustly. Essentially, a lower SSD minimizes the complexity term in the risk bound, meaning that models with sparser, more compressible internal representations are theoretically guaranteed to have a tighter generalization gap.}

\subsection{Operational Measurement of Complexity Parameters}

While our theoretical framework relies on the existence of an optimal concept pool $G^*$ and its cardinality $P$, these quantities can be robustly approximated empirically. We provide operational definitions below to facilitate experimental verification.

\paragraph{Measuring the Concept Pool Size $\hat{P}$.}
The theoretical dimension $P = |G^*|$ represents the size of the minimal subset of dictionary atoms required to cover the data distribution with high probability. In practice, we estimate this via the \textit{effective vocabulary size} of the SAE on the training set. Let $\mathcal{D}_{train}$ be the training corpus. Assume the number of examples in $\mathcal{D}_{train}$ as $N_{cal}$. We define the empirical pool $\hat{G}$ as the set of all features that activate at least once (or exceed a frequency threshold $\tau$) across the corpus:
\begin{equation}
    \hat{G} := \bigcup_{x \in \mathcal{D}_{train}} \text{supp}(\text{TopK}(S_E(M(x))))
\end{equation}
The empirical Sparse Semantic Dimension is then estimated as $\widehat{SSD} = \hat{P} \log(em/\hat{P})$, where $\hat{P} = |\hat{G}|$. This quantity serves as a concrete, measurable proxy for the theoretical model complexity. Note that although we take union over all points in the calibration set to find $P$, the theory allows us to choose a $P \ll m$ by plotting bound vs P plot to check the strictest bound and use that corresponding P subsequently while evaluation. Note that we also use empirical $\eta$ as $\hat{\eta}$, and that would require a very small hoeffding's penalty to the bound in Equation \ref{eq:occam_final_simplified}, but we ignore it due to its independence on any active parameter and a very small value.

\paragraph{Constructing the Pool-Restricted Predictor $h_{G^*}$.}
To evaluate the risk of the restricted predictor $h_{G^*}$ (as required by Theorem \ref{thm:occam_pool}), we implement it as a masked intervention on the forward pass. For a fixed pool $G^*$ (or its estimate $\hat{G}$), the predictor $h_{G^*}$ is defined operationally as follows:
i) \textbf{Forward Pass:} Run the base model $M$ on input $x$ up to the target layer to obtain hidden state $h$.
ii) \textbf{Encode:} Compute the pre-activation vector $a = S_E(h) \in \mathbb{R}^m$.
iii) \textbf{Pool Masking:} Apply the pool constraint by zeroing out all indices not in $G^*$:
    \begin{equation}
        a_{G^*} = a \odot \mathbf{1}_{G^*}
    \end{equation}
    where $\mathbf{1}_{G^*}$ is the indicator vector for the pool.
iv) \textbf{Select Features:} Perform the Top-K operation on the masked vector: $c_{G^*} = \text{TopK}(a_{G^*})$.
v) \textbf{Reconstruct \& Resume:} Decode $\hat{h}_{G^*} = S_D(c_{G^*})$ and inject it back into the model, completing the forward pass to obtain logits.

This procedure allows us to directly measure the empirical risk $\hat{R}_S(h_{G^*})$ and verify the pool mismatch error $\epsilon_{pool}$ by comparing it against the unrestricted proxy $S \circ M$.

\section{Experimental Validation}

To empirically validate \textbf{Theorem \ref{thm:occam_pool}}, we conduct two experiments linking the theoretical Sparse Semantic Dimension (SSD) to observe the model behavior. The first experiment verifies that our generalization bound, based on the \textit{global concept pool} $P$, is non-vacuous and tightens with sample size. The second experiment tests the core implication of the Concept Pool Assumption: that violations of the learned feature pool serve as a reliable proxy for epistemic uncertainty under distribution shifts. 
To evaluate the model's behavior under distinct distributional shifts, we utilize three datasets:

\begin{itemize}
    \item \textbf{In-Distribution (ID):} We use the \textbf{C4 dataset (\url{https://huggingface.co/datasets/allenai/c4})}, comprising standard English text, to represent the distribution on which the LLM is trained.
    \item \textbf{Shifted Distribution (Code):} We employ the \textbf{GitHubCode dataset} (\url{https://huggingface.co/datasets/codeparrot/github-code}) by CodeParrot to represent a structured but distributionally shifted domain. While these samples possess syntactic structure, they exhibit lower likelihood under the LLM model compared to natural language.
    \item \textbf{Far Out-Of-Distribution (Far-OOD):} We generate sequences of \textbf{random tokens} to simulate a noisy input (out-of-distribution samples) regime. This serves as a baseline to measure model behavior when the input lacks all semantic structure.
\end{itemize}

\begin{figure}
    \centering
    \includegraphics[width=0.85\linewidth]{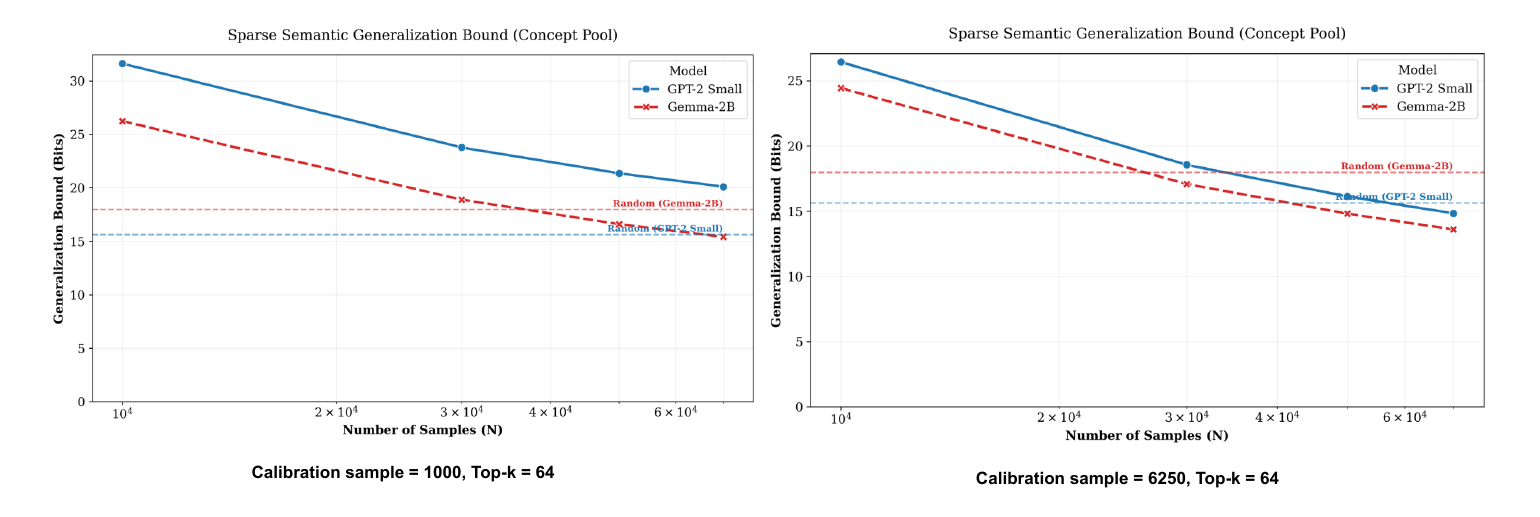}
  \caption{\textbf{Sensitivity of Generalization Bounds to Concept Pool Calibration.} We plot the theoretical generalization certificate (in bits) against evaluation sample size $N$ for GPT-2 Small (blue) and Gemma-2B (red). The left panel utilizes a minimal calibration set ($N_{cal}=1000$), while the right panel utilizes a moderate calibration set ($N_{cal}=6250$).}
    \label{fig:convergence}
\end{figure}

\subsection{Non-Vacuous Generalization Bounds}

\textbf{Motivation.} Standard generalization bounds (e.g., VC-dimension, Rademacher complexity of weights) are vacuous for deep learning, often predicting error rates $>100\%$ due to the massive parameter count $P_\theta$. Our framework shifts the complexity measure to the \textit{active concept pool size} $P < m$, representing the effective vocabulary of features used by the model. We aim to demonstrate that the resulting bound (Eq. \ref{eq:occam_final_simplified}) is non-vacuous (i.e., tighter than a random guess) for real-world LLMs.

\textbf{Setup.} We evaluate two models: \texttt{GPT-2 Small} \citep{radford2019language} (124M parameters) and \texttt{Gemma-2B} \citep{gemmateam2024gemma2improvingopen} (2B parameters). For each model, we use a pre-trained Sparse Autoencoder (SAE) to decompose activations (Layer 6 for GPT-2, Layer 12 for Gemma). Sequence length is fized at 32 tokens. To rigorously test the validity of the certificate and its sensitivity to the learned manifold, we evaluate the bound under two distinct calibration regimes with fixed sparsity ($Top\text{-}k=64$):

i) \textbf{Low-Data Calibration ($N_{cal}=1,000$):} A "few-shot" setting where the active concept pool $G^*$ is estimated using only $\approx 1000$ sequences (Left Panel of Figure \ref{fig:convergence}). This tests how quickly the model's semantic manifold is revealed.
ii) \textbf{Moderate-Data Calibration ($N_{cal}=6,250$):} A more robust setting where the pool is estimated using $6250$ sequences (Right Panel of Figure \ref{fig:convergence}), testing the bound's stability as the dictionary coverage improves.

\textbf{Results and Analysis.} Figure \ref{fig:convergence} illustrates the convergence of our bound.

i) \textbf{Non-Vacuousness Verification:} The generalization bound drops below the random baseline ($15.6$ bits for GPT-2, $18.0$ bits for Gemma-2B) at realistic sample sizes, provided sufficient calibration. In the moderate calibration regime (Right Panel), both models achieve non-vacuousness: \texttt{Gemma-2B} crosses the baseline at $N \approx 2.8 \times 10^4$, and \texttt{GPT-2 Small} crosses at $N \approx 5.5 \times 10^4$. This confirms that the effective dimension of the feature space ($m \approx 10^4$) provides a mathematically valid certificate of complexity where parameter counting ($P_\theta \approx 10^9$) leads to vacuous bounds. Also note how Gemma-2 even being a larger model generalized better than smaller model of GPT-2.

ii) \textbf{Sensitivity to Calibration (Manifold "Sharpness"):} A striking divergence appears in the low-data regime (Left Panel, $N_{cal}=1,000$).  \textbf{Gemma-2B} remains robust, achieving non-vacuousness ($N \approx 3.2 \times 10^4$) even when the concept pool is learned from minimal data.  \textbf{GPT-2 Small}, conversely, fails to certify in this regime. Its bound remains vacuous (hovering $\approx 20$ bits vs. baseline $15.6$) due to a high pool mismatch rate ($\eta$). It requires the larger calibration set (Right Panel) to stabilize.

iii) \textbf{Efficiency of Larger Models:} This result reinforces a counter-intuitive scaling property. Despite being an order of magnitude larger ($2B$ vs $124M$ parameters), \texttt{Gemma-2B} is structurally \textit{simpler} to certify. It requires fewer calibration samples to define its active concept pool and achieves non-vacuousness with fewer evaluation samples ($2.8 \times 10^4$ vs $5.5 \times 10^4$). We attribute this to the superior compressibility of larger models: they utilize a more consistent, structured dictionary that is efficiently approximated by the SAE, whereas smaller models exhibit "noisier" feature activations that require more data to cover.

\begin{table}[t]
    \centering
    
    \resizebox{\textwidth}{!}{%
    \begin{tabular}{lcccccccc}
        \toprule
        \textbf{Model} & \textbf{Calibration} & \textbf{Dict. Size} & \textbf{Active Pool} & \textbf{Empirical Risk} & \textbf{Rec. Gap} & \textbf{Mismatch} & \textbf{Bound} & \textbf{Status} \\
        & ($N_{cal}$) & ($m$) & ($P$) & ($\hat{R}_S$) & ($\hat{\epsilon}_{loss}$) & ($\eta$) & (Bits) & (vs. Base) \\
        \midrule
        \textbf{GPT-2 Small} & 200k & 24,576 & 24,399 & 7.37 & 0.22 & 0.012 & \textbf{14.84} & Non-Vacuous \\
        (Base: 15.62) & 32k & 24,576 & 24,121 & 7.37 & 0.22 & 0.329 & 20.11 & Vacuous \\
        \midrule
        \textbf{Gemma-2B} & 200k & 16,384 & 16,001 & 6.43 & 0.50 & 0.004 & \textbf{13.60} & Non-Vacuous \\
        (Base: 17.97) & 32k & 16,384 & 15,852 & 6.44 & 0.50 & 0.100 & \textbf{15.42} & Non-Vacuous \\
        \bottomrule
    \end{tabular}%
    }
    \caption{\textbf{Sensitivity of Generalization Bounds to Calibration Size ($N_{cal}$).} We report the generalization bound (in bits) for GPT-2 Small and Gemma-2B under two calibration regimes ($N_{cal} = 200k$ vs. $32k$) with a fixed evaluation sample size of $N=70,000$. The baseline represents the random guess log-loss ($\log_2 V$). $\eta$ denotes the pool mismatch rate, and $P$ represents the effective concept pool size.}
    \label{tab:calibration_sensitivity}
\end{table}

\textbf{Analysis of Calibration Sensitivity.} Table \ref{tab:calibration_sensitivity} demonstrates the critical dependency of the generalization certificate on the size of the calibration set ($N_{cal}$). We observe that increasing the number of calibration tokens leads to the discovery of a marginally larger active pool $P$. This indicates that the dataset contains a vast array of distinct semantic concepts; owing to the monosemanticity of SAE features, we can assume these additional active features represent distinct, albeit rare, directions in the activation manifold.

However, the most significant impact of reducing calibration size is observed in the pool mismatch rate, $\eta$. As $N_{cal}$ decreases from $200k$ to $32k$, $\eta$ increases for both models, introducing a larger penalty term to the structural risk minimization bound.

A striking dichotomy emerges when comparing the two architectures in the low-data regime ($N_{cal}=32k$):

i) \textbf{GPT-2 Small:} The model exhibits a high sensitivity to calibration size. At $32k$ tokens, the mismatch rate spikes to $\eta \approx 0.329$, causing the total generalization bound (20.11 bits) to exceed the random baseline (15.62 bits). Consequently, the certificate becomes \textit{vacuous} due to the extreme value of $\eta$, implying that a small calibration set fails to cover the semantic manifold of GPT-2 efficiently.
    
ii) \textbf{Gemma-2B:} Conversely, despite being an order of magnitude larger in parameter count ($2B$ vs $124M$), Gemma-2B maintains a \textit{non-vacuous} generalization bound (15.42 bits $<$ 17.97 bits) even with only 32k calibration tokens (approx. 1,000 sequences). The mismatch rate remains remarkably low ($\eta=0.100$), suggesting that Gemma's internal representations are organized into a more concise, compressible dictionary.

This result reinforces the hypothesis that "compression leads to generalization" \citep{arora2018strongergeneralizationboundsdeep}. Gemma-2B's ability to achieve a valid generalization certificate with a sparser concept pool compared to GPT-2, discovered using minimal calibration data—indicates that larger, more capable models do not necessarily require more complex descriptions. Instead, they appear to learn a "sharper" semantic dictionary that covers the data distribution more efficiently than smaller models.

\subsection{Analysis of Bound Components under IID vs Noise regime}

\textbf{Motivation.} A robust generalization bound must differentiate between learnable structure (such as natural language) and unlearnable noise. To validate this, we decompose the total generalization bound (Theorem 3) into its constituent terms: Empirical Risk ($\hat{R}_S$), Reconstruction Gap ($\hat{\epsilon}_{loss}$), Pool Mismatch ($\eta B$), and the Complexity Penalty ($\Omega \propto \sqrt{SSD}$). We compare these terms on In-Distribution data (English) versus Far-OOD Random Noise (random tokens).

\textbf{Experimental Setup.} We evaluate the decomposition on \texttt{GPT-2 Small} and \texttt{Gemma-2B} using the strict held-out separation detailed in Table \ref{tab:experiment_results} (Panel A). For this analysis, we utilize a standard sparse inference setting ($Top\text{-}k=64$) to observe how the active concept pool behaves under distribution shifts.

\begin{table}[ht]
    \centering
    
    \begin{tabular}{lcc}
        \multicolumn{3}{c}{\textbf{Panel A: Experimental Hyperparameters for all experiments}} \\
        \toprule
        \textbf{Parameter} & \textbf{Value} & \textbf{Description} \\
        \midrule
        Calibration seqs. ($N_{cal}$) & 6250 & Samples to learn the concept pool $G^*$ \\
        Test seqs. ($N_{test}$) & 70,000 & Samples to measure Risk \& Gap \\
        Context Length ($T$) & 32 & Sequence length for evaluation \\
        Smoothing ($\alpha$) & 0.5 & Probability smoothing factor \\
        Confidence ($\delta$) & 0.05 & Bounds holds with prob $1-\delta$ \\
        Batch Size & 16 & batch size for models \\
        \bottomrule
    \end{tabular}
    \vspace{0.3cm}
    
    \begin{tabular}{l|ccccc|c}
        \multicolumn{7}{c}{\textbf{Panel B: Quantitative Decomposition (Bits per Dimension)}} \\
        \toprule
        \textbf{Model / Condition} & \textbf{Risk} ($\hat{R}_S$) & \textbf{Gap} ($\hat{\epsilon}$) & \textbf{Mismatch} ($\eta B$) & \textbf{Comp.} ($\Omega$) & \textbf{Total} & \textbf{P} \\
        \midrule
        \textit{Gemma-2B ($k=64$)} & & & & & & \\
        English (IID) & 6.39 & 0.50 & 0.12 & 6.49 & \textbf{13.60} & 
        15985 \\

        Noise (Far-OOD) & 18.78 & 0.14 & 1.67 & 6.48 & \textbf{27.17} & 15114 \\
        \midrule
        \textit{GPT-2 Small ($k=64$)} & & & & & & \\
        English (IID) & 7.35 & 0.22 & 0.23 & 6.96 & \textbf{14.86} & 24390 \\
        Noise (Far-OOD) & 16.19 & 0.06 & 1.93 & 6.95 & \textbf{25.23} & 22901 \\
        \bottomrule
    \end{tabular}
    \caption{\textbf{Experimental Analysis of Bound Components.} \textbf{Panel A} details the hyperparameters used for the decomposition. \textbf{Panel B} presents the quantitative breakdown of the bound (in bits) for Gemma-2B and GPT-2, showing that OOD rejection is driven by maximal empirical risk and a spike in pool mismatch.}
    \label{tab:experiment_results}
\end{table}

\textbf{Analysis.} As illustrated in Figure \ref{fig:bound_decomposition} and quantified in Table \ref{tab:experiment_results}, we observe two distinct regimes of generalization:

1. \textbf{The Generalization Regime (IID):} For natural language, the bound certifies generalization through a cooperative balance. Both models achieve low empirical risk ($\hat{R}_S\ \approx 6.39$ bits for Gemma, $\approx 7.35$ bits for GPT-2)\footnote{$\hat{R}_S$ is same as $\hat{R}_S(h_{G^*})$} and a minimal reconstruction gap. Crucially, the \textbf{Pool Mismatch} term ($\eta B$) is negligible ($0.12$ bits for Gemma, $0.23$ bits for GPT-2), indicating that the semantic concepts learned during the calibration phase ($N_{cal}=6250$) successfully cover the features activated by unseen test data. Consequently, the total bound ($13.60$ bits for Gemma, $14.86$ bits for GPT-2) drops below the random baseline ($\log_2 V$), certifying non-vacuous generalization.

2. \textbf{The Noise Regime (Far-OOD):} In the random noise setting, the bound correctly identifies the data as unlearnable, with the total bound exploding to vacuous levels ($\approx 25-27$ bits). This rejection is driven by two key factors:i) \textbf{Irreducible Empirical Risk:} The risk maximizes to the vocabulary entropy ($\approx 18.78$ bits for Gemma, $\approx 16.19$ bits for GPT-2), dominating the bound. ii) \textbf{Random sparsity:} We observe a significant relative spike in the Pool Mismatch term. For \texttt{Gemma-2B}, this term increases by over $10\times$ ($0.12 \rightarrow 1.67$ bits), and for \texttt{GPT-2}, it rises similarly ($0.23 \rightarrow 1.93$ bits). This confirms that while noise inputs activate features, they trigger a "hallucination" of concepts outside the learned semantic pool $G^*$, effectively violating the sparse dictionary assumption.

3. \textbf{Reconstruction Collapse:} We note that the Reconstruction Gap ($\hat{\epsilon}$) paradoxically drops to near zero for noise ($\approx 0.14$ bits for Gemma); this is not a sign of quality, but rather indicates that both the dense model ($M$) and the sparse proxy ($S \circ M$) have collapsed to uniform random guessing, resulting in zero disagreement between their predictions.

\begin{figure}[t]
    \centering
    \includegraphics[width=\linewidth]{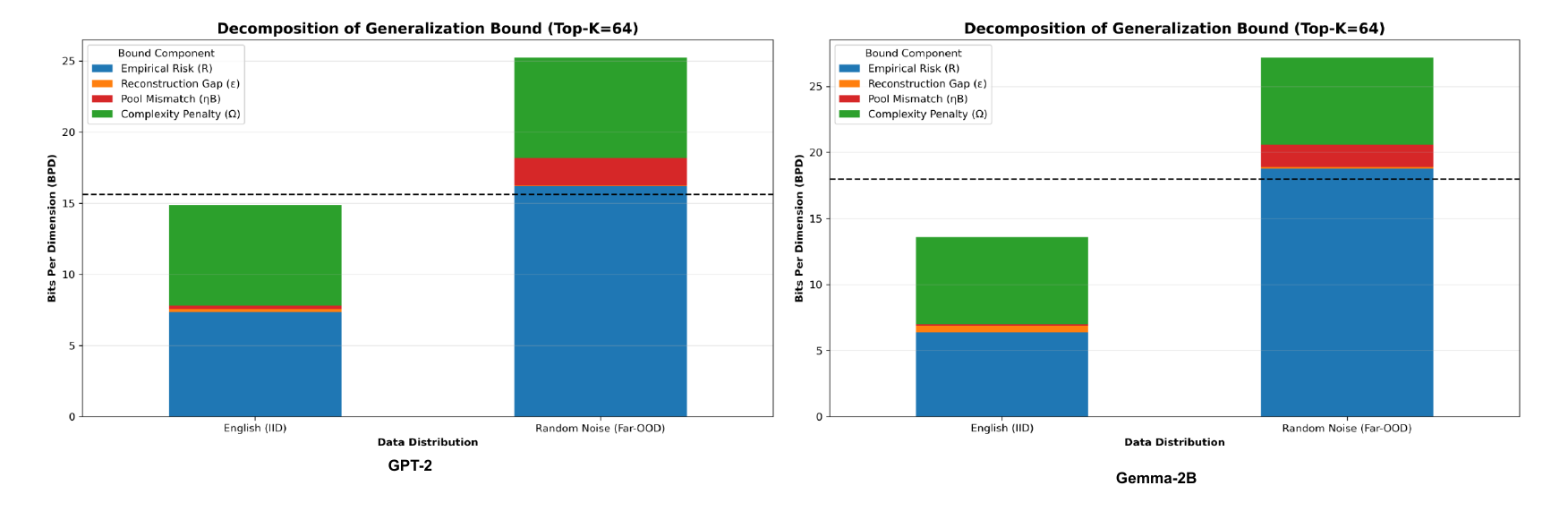}
    \caption{\textbf{Decomposition of the Generalization Bound Components.} We visualize the contribution of Risk ($R$), Gap ($\epsilon$), Mismatch ($\eta B$), and Complexity ($\Omega$) to the total bound. The plots compare In-Distribution (English) vs. Far-OOD (Random Noise) under the constraint Top-k=64. The bound certifies generalization for English via low risk and low pool mismatch. Conversely, it rejects Noise due to maximal empirical risk and a massive spike in pool mismatch (indicated in red).}
    \label{fig:bound_decomposition}
\end{figure}

\subsection{Runtime Uncertainty Quantification via Feature Density}

\textbf{Motivation.} Although Theorem \ref{thm:occam_pool} establishes that generalization is governed by the global concept pool size $P$, measuring $P$ requires a full data set pass. For real-time data, we require a local proxy. We hypothesize that the \textit{instantaneous active feature count} $k(x) = \|S_E(x)\|_0$ serves as this proxy. A spike in $k$ (feature explosion) indicates that the input $x$ cannot be sparsely represented by the learned dictionary, signaling a local violation of the sparsity assumption and, consequently, we hypothesize that this leads to high epistemic uncertainty. 

\textbf{Setup.} We probed \texttt{GPT-2 Small} and \texttt{Gemma-2B} in three regimes representing samples from three different distributions:
i) \textbf{In-Distribution (ID):} Standard English text (C4).
ii) \textbf{Shifted Structure:} Python Code (CodeParrot). A structured domain distinct from the training distribution (for GPT-2).
iii) \textbf{Maximum Entropy:} Random Tokens. Representing maximal epistemic uncertainty. We use 2240000 tokens (i.e. 70000 examples) from all datasets to perform the experiments.

\textbf{Results.}
As shown in Figure \ref{fig:ood_shift} and Table \ref{tab:ood_stats}, runtime sparsity $k$ acts as a highly sensitive distribution detector, although its behavior varies by model scale.

\begin{figure*}[t]
    \centering
    \includegraphics[width=0.85\textwidth]{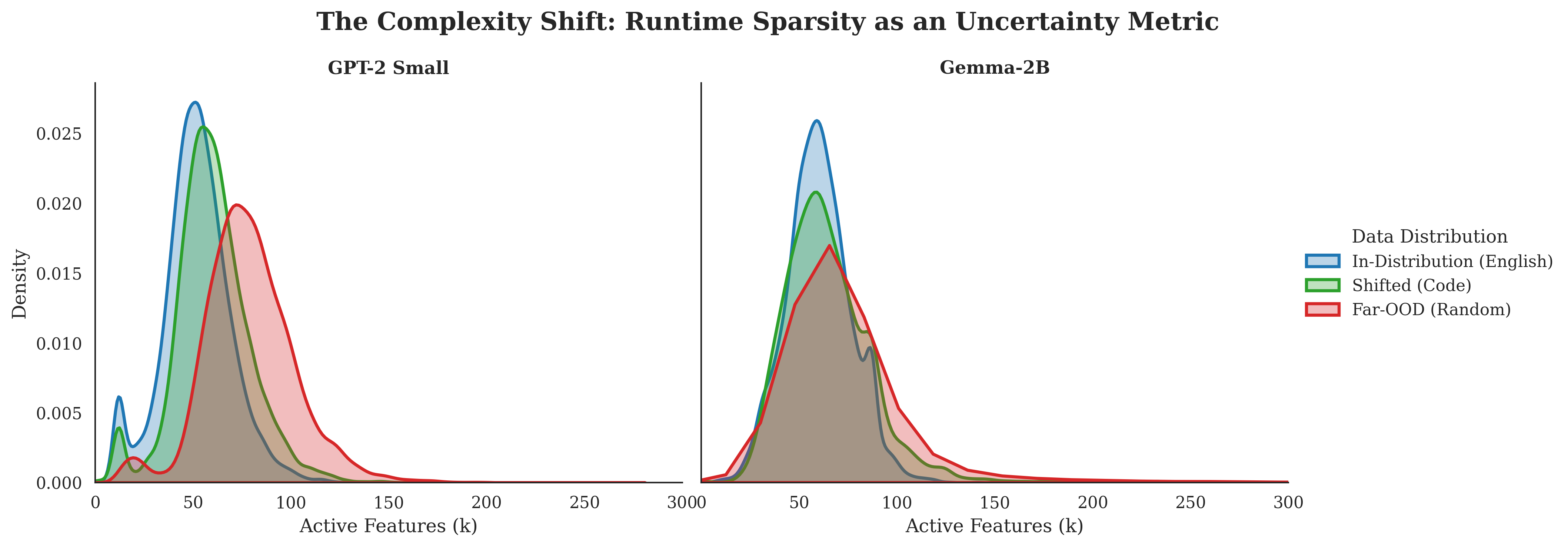}
    \caption{\textbf{The Complexity Shift.} Histograms of active feature counts ($k$) for GPT-2 and Gemma-2B. GPT-2 shows a clear distributional shift to the right (higher complexity) for noise. Gemma-2B reveals a "heavy-tailed" failure mode on Far-OOD data (Max $k > 1500$) and a "compression" mode on Code (Shift Left), accurately reflecting its training distribution alignment.}
    \label{fig:ood_shift}
\end{figure*}


\begin{table}[ht]
\centering
\begin{tabular}{l|ccc|ccc}
\toprule
 & \multicolumn{3}{c|}{\textbf{GPT-2 Small}} & \multicolumn{3}{c}{\textbf{Gemma-2B}} \\
\textbf{Condition} & Mean $k$ & Std & Max $k$ & Mean $k$ & Std & Max $k$ \\ \midrule
In-Distribution (English) & 52.02 & 17.69 & 165 & 60.31 & 16.46 & 237 \\
Shifted (Code) & 59.98 & 19.31 & 203 & 64.32 & 22.59 & 382 \\
Far-OOD (Random) & \textbf{78.89} & 23.58 & 419 & 78.33 & 93.69 & \textbf{3840} \\ \bottomrule
\end{tabular}
\caption{Sparsity statistics across distributions. Note the massive max-k spike for Gemma on OOD data.}
\label{tab:ood_stats}
\end{table}

1. \textbf{ID data:} For \texttt{GPT-2}, OOD noise triggers a "Feature Explosion," shifting mean sparsity from $52.02 \to 78.89$. The same happens for Gemma-2B as well. The model attempts to approximate unstructured noise by "superposing" unrelated semantic dictionary features. This confirms that high $k$ serves as a direct proxy for confusion—when the model is input a OOD input, it activates a dense combination of features.
    
2. \textbf{Shifted ID data:} Like GPT-2, Gemma also exhibits \textit{lower} sparsity on code ($k=64.32$) than on English ($k=60.31$). For both models, mean $k$ value is higher for code data compared to in-distribution English data.
    
3. \textbf{Catastrophic Feature Collapse:} On Far-OOD data, \texttt{Gemma-2B} exhibits a large tail behavior. While the mean shift is moderate, the maximum $k$ explodes to \textbf{3840} (a $12\times$ spike vs. ID). This signals a catastrophic failure of the sparse representation. We may propose rule such as a simple runtime guardrail (e.g., $k > 500$) could detect 100\% of these sparsity violations with near-zero false positives on natural text.

\subsection{Does specific features matter for generalization?}

\textbf{Hypothesis.} Even with a learned dictionary, does the \textit{specific choice} of features matter? We test this by breaking the semantic link between feature indices and their meanings. If the bound relies only on the statistical properties of the activation vector (e.g., norms and sparsity count), shuffling the features should have no effect.

\textbf{Methodology.} We implemented a "Shuffled Feature" intervention: for every token, we take the exact activation vector $c(x)$ produced by the SAE (preserving exact $L_0$ sparsity and $L_2$ norm) but permute the indices randomly before decoding. We then measure the distribution of the reconstruction gap ($\epsilon_{loss}$) for both the real and shuffled conditions.

\begin{figure}[ht]
    \centering
    \includegraphics[width=1.0\textwidth]{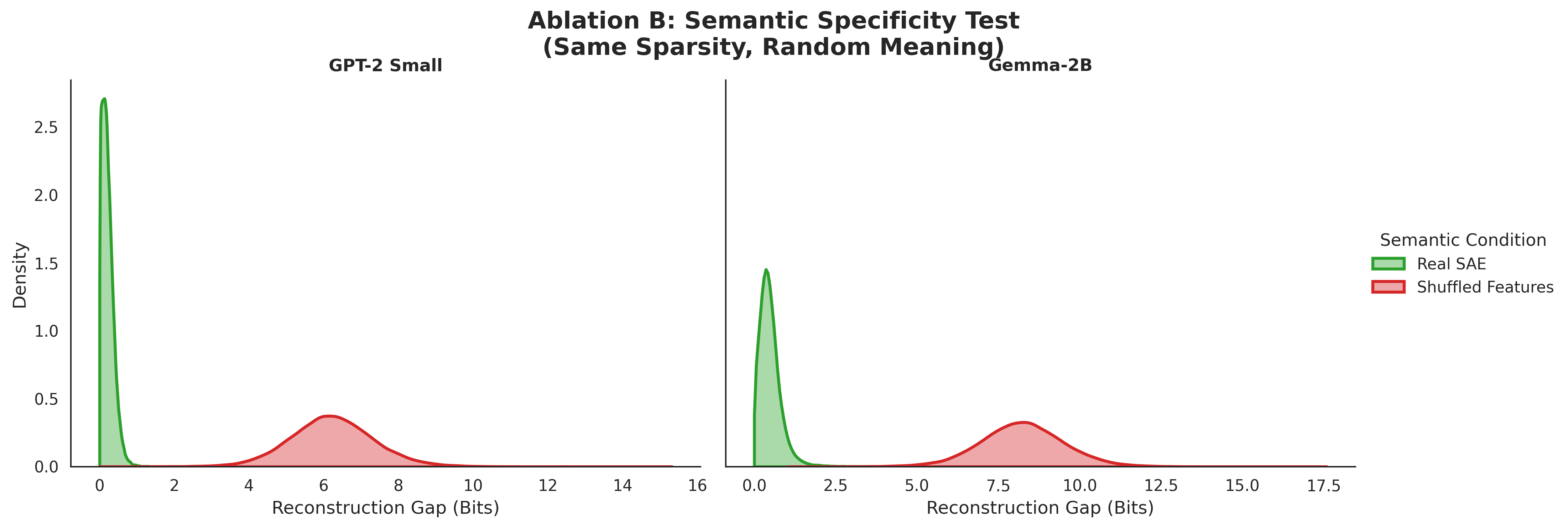}
    \caption{\textbf{Ablation: Semantic Specificity.} Histograms of the per-sequence reconstruction gap ($\epsilon_{loss}$). \textbf{Green (Real SAE):} The error is tightly clustered near 0 bits, indicating high semantic fidelity. \textbf{Red (Shuffled):} Permuting the feature indices—while maintaining identical per-sample sparsity $k$—causes the error distribution to shift distinctively to the right, adopting a \textbf{Gaussian-like profile} (Mean shift: $\approx 6.5$ bits for GPT-2, $\approx 8.5$ bits for Gemma). This proves that the generalization guarantee depends on the specific \textit{semantic alignment} of the activated features, not just statistical sparsity.}
    \label{fig:ablation_shuffled}
\end{figure}

\textbf{Results.} Figure \ref{fig:ablation_shuffled} illustrates a catastrophic failure of the proxy predictor under shuffling.

\begin{enumerate}
    \item \textbf{Error Explosion:} For the Real SAE (Green), the reconstruction gap is minimal and tightly clustered near 0 bits, indicating high fidelity. In contrast, for the Shuffled condition (Red), the error distribution explodes and \textbf{mimics a Gaussian curve}, shifting significantly to the right. The mean error rises to $\approx 6.5$ bits for GPT-2 and $\approx 8.5$ bits for Gemma-2B, representing the loss of semantic information.
    
    \item \textbf{Theoretical Implication:} Since the complexity term ($SSD$) remains identical for both conditions (as sparsity $k$ is preserved), the difference in the bound is driven entirely by the reconstruction penalty $\epsilon_{loss}$. This confirms that our theorem is sensitive to \textit{semantics}: it credits the model only when it finds a sparse representation that is strictly aligned with the data distribution.
\end{enumerate}

\section{Conclusion and Future Work}

In this work, we addressed the Generalization Paradox of Large Language Models (LLMs), resolving the discrepancy between their massive parameter counts ($P_{\theta} \gg N$) and their robust ability to generalize by proposing that the effective complexity is defined by the \textit{Sparse Semantic Dimension} (SSD) of the activation manifold. We derived a high-probability risk bound that scales with the size of the active feature pool $P$ rather than the model parameter scale.
Empirically, we validated this bound on GPT-2 Small and Gemma-2B, demonstrating that it becomes non-vacuous at realistic sample sizes ($N \approx 25,000$), a regime where parameter-counting bounds would fail and establishing that oer-sample sparsity of the learned dictionary features of the SAE (termed as $k$) reliably signal epistemic uncertainty.

Our findings may open several avenues for future research at the intersection of mechanistic interpretability and learning theory, particularly regarding the scaling laws of sparse semantics. While we observed that larger models like Gemma-2B exhibit generalization n realistic data regime, it remains to be seen how the ratio of SSD to parameter count evolves for frontier-scale models (e.g., 70B+ parameters) and whether a ``sparsity saturation'' point exists. Future work should also investigate layer-wise generalization dynamics (by using SAE at different layers) to determine which layers govern memorization versus generalization. Finally, extending our theoretical bounds to account for dynamic or context-dependent concept pools could provide tighter guarantees for models operating in rapidly shifting domains, generalizing the fixed-pool assumption used in this initial framework.

\bibliography{main}
\bibliographystyle{tmlr}

\end{document}